\documentclass[conference]{IEEEtran}
\IEEEoverridecommandlockouts
\usepackage{cite}
\usepackage{amsmath,amssymb,amsfonts}
\usepackage{algorithmic}
\usepackage{graphicx}
\usepackage{textcomp}
\usepackage{xcolor}

\usepackage{cite}
\usepackage{amsmath,amssymb,amsfonts}
\usepackage{algorithmic}
\usepackage{graphicx}
\usepackage{textcomp}
\usepackage{xcolor}
\usepackage{hyperref}
\usepackage{multirow}
\usepackage{times}
\usepackage{booktabs}       % professional-quality tables
\usepackage{soul}
\usepackage{url}
\usepackage{siunitx} 
\usepackage{graphicx}
\usepackage{amsmath}
\usepackage{amsthm}
\usepackage{algorithm}
\usepackage{algorithmic}
\usepackage{wrapfig}

\newtheorem{theorem}{Theorem}
\newtheorem{definition}{Definition}

\newtheorem{lemma}{Lemma}

\def\BibTeX{{\rm B\kern-.05em{\sc i\kern-.025em b}\kern-.08em
    T\kern-.1667em\lower.7ex\hbox{E}\kern-.125emX}}
\begin{document}

\title{Modeling Time Series Dynamics with Fourier
Ordinary Differential Equations\\
}

\author
{\IEEEauthorblockN{Muhao Guo}
\IEEEauthorblockA{\textit{School of Electrical, Computer and Energy Engineering} \\
\textit{Arizona State University}\\
Tempe, United States \\
mguo26@asu.edu}
\and
\IEEEauthorblockN{Yang Weng}
\IEEEauthorblockA{\textit{School of Electrical, Computer and Energy Engineering} \\
\textit{Arizona State University}\\
Tempe, United States\\
yang.weng@asu.edu}
% \and
% \and
% \and
% \and
}

\maketitle

\begin{abstract}
Neural ODEs (NODEs) have emerged as powerful tools for modeling time series data, offering the flexibility to adapt to varying input scales and capture complex dynamics. However, they face significant challenges: first, their reliance on time-domain representations often limits their ability to capture long-term dependencies and periodic structures; second, the inherent mismatch between their continuous-time formulation and the discrete nature of real-world data can lead to loss of granularity and predictive accuracy. To address these limitations, we propose Fourier Ordinary Differential Equations (FODEs), an approach that embeds the dynamics in the Fourier domain. By transforming time-series data into the frequency domain using the Fast Fourier Transform (FFT), FODEs uncover global patterns and periodic behaviors that remain elusive in the time domain. Additionally, we introduce a learnable element-wise filtering mechanism that aligns continuous model outputs with discrete observations, preserving granularity and enhancing accuracy. Experiments on various time series datasets demonstrate that FODEs outperform existing methods in terms of both accuracy and efficiency. By effectively capturing both long- and short-term patterns, FODEs provide a robust framework for modeling time series dynamics.
\end{abstract}


\section{Introduction}
Many time series data exhibit long-term trends and periodic patterns spanning the entire dataset \cite{weng2022transform}.
Traditional methods often fail to capture such periodic and extended temporal dependencies, leading to suboptimal predictions. These challenges arise across diverse domains, such as energy \cite{li2025exarnn} and healthcare \cite{guo2022patients}, further motivating the need for flexible and expressive approaches capable of uncovering both global and local structures in time series data. 

Neural ODEs ~\cite{chen2018neural} offer a powerful alternative to discrete-layer models through their continuous-depth formulation. Instead of stacking fixed, discrete layers, NODEs treat the hidden representation as evolving along a continuous trajectory, modeled by an ordinary differential equation (ODE). Concretely, let $h(t) \in \mathbb{R}^N$ denote the hidden state at time $t$. NODEs use a parameterized function $f(h(t), t, \theta)$, often instantiated as a neural network with learnable parameters $\theta$, to describe the time evolution of $h(t)$. Formally, one writes
$\frac{d h(t)}{dt} = f\bigl(h(t), t, \theta \bigr).$
Given an initial value $h(t_0)$, the hidden state at any future time $t_1$ can be computed via
$h(t_1) = h(t_0) + \int_{t_0}^{t_1} f_{\theta}\bigl(h(t), t \bigr)\,dt \nonumber = \text{Solver}\bigl(h(t_0), f_{\theta}, t_0, t_1 \bigr)$.
This flexibility enables NODEs to adapt to varying input scales and shapes, balancing numerical precision with computational efficiency~\cite{kidger2022neural, xia2021heavy, guo2023continuous, dupont2019augmented}. Similar considerations—such as transparent reasoning and evaluation of failure modes—are increasingly emphasized in LLM-driven systems \cite{guo2024bias, guo2025solar}.

\begin{figure*}[t]
\centering 
\includegraphics[width=1.8\columnwidth]{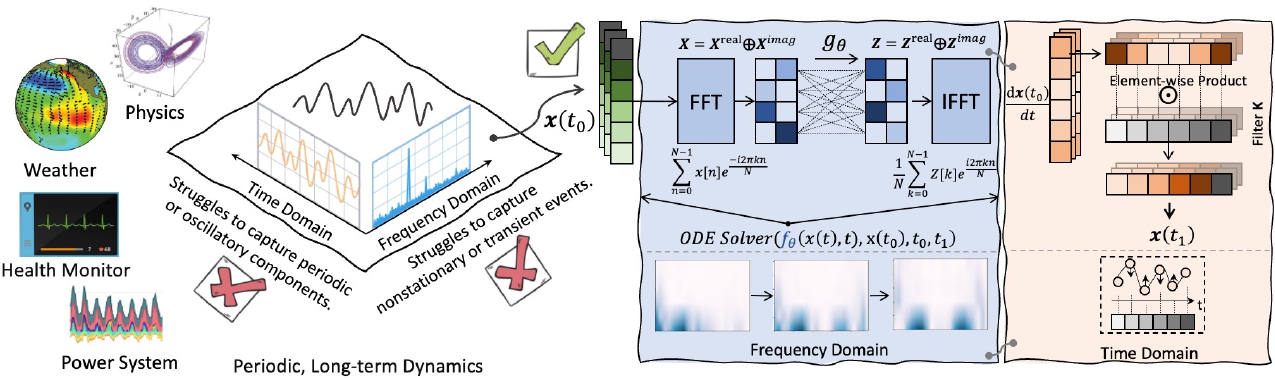}
\caption{Schematic of the proposed method. 
Blue region: The input time series \(x(t_0)\) is first transformed to the frequency domain via FFT, where a neural operator learns the dynamics. An inverse FFT (IFFT) then maps the representation back to the time domain. 
Orange region: A learnable element-wise filter 
\(K\) refines the final prediction \(x(t_1)\). This design leverages both frequency and time-domain operations to capture complex patterns in the data.}
\label{structure}
\end{figure*}

Despite these advantages, NODEs face two primary limitations when modeling time series data. First, because NODEs typically learn representations in the time domain, they can struggle to capture global structures that evolve over long intervals or across multiple frequencies. While the time-domain perspective is effective at modeling local temporal dynamics, it may not fully reveal the broad, periodic, or long-range dependencies that are crucial in many applications. Second, although NODEs excel at continuous-time generalization, real-world time series data are sampled discretely \cite{guo2023msq, guo2024bayesian, guo2023graph}. This misalignment between continuous modeling and discrete observations can lead to a loss of granularity and inaccuracies when reconstructing the original signals. 

To address the first limitation, analyzing time series data in the Fourier domain provides a powerful way to identify and model global structures~\cite{zhou2022fedformer}.  
The Fourier transform decomposes a signal into its constituent frequencies, enabling the discovery of dominant periodic components and long-range dependencies. Such frequency-based representations often uncover patterns that remain hidden in purely time-based approaches. Recent work~\cite{lee2021fnet, zhou2022fedformer} demonstrates that incorporating Fourier analysis into deep architectures can improve generalization and capture global relationships, with applications ranging from transformer-based networks~\cite{vaswani2017attention} to language modeling on the GLUE benchmark~\cite{guo2024transparent, wang2018glue}.

For the second limitation, reconciling continuous modeling with discrete observations—element-wise filtering (via the Hadamard product) proves effective~\cite{horn1990hadamard}. 
By multiplying continuous model outputs with data-derived correction factors or masks, this filtering step enforces alignment between discrete samples and continuous trajectories. It also highlights important features, amplifying relevant parts of the signal while suppressing noise or less critical components. This operation bridges the gap between the NODE's continuous formulation and the discrete timestamps in real-world time series, ensuring both flexibility and fidelity to the observed data. 

Therefore, we propose Fourier Ordinary Differential Equations (FODEs) to capture global patterns and preserve the discrete granularity of time series data (see Figure~\ref{structure}). FODEs embed the hidden representations in the Fourier domain, leveraging frequency decompositions to reveal long-term dependencies and periodic behaviors. At the same time, element-wise filtering refines the continuous outputs, aligning them with the inherent discrete nature of real-world observations. This hybrid approach unites the strengths of both frequency and time-based modeling, offering an effective solution for complex time series tasks such as forecasting, classification, and anomaly detection.

\section{Methodology}
\subsection{Discrete Fourier Transform}
We commence by introducing the Discrete Fourier Transform (DFT) \cite{winograd1978computing}, a fundamental tool in digital signal processing. The DFT is applied to a sequence of $N$ complex numbers $x \lbrack n \rbrack$, where $0 \leq n \leq N-1$, to convert it into the frequency domain. The 1D DFT is defined as follows:
\begin{equation}
X \lbrack k \rbrack = \sum^{N-1}_{n=0} x\lbrack n \rbrack e^{-i 2\pi k n / N},
\end{equation}
where $i$ denotes the imaginary unit. The DFT maps the input sequence $x \lbrack n \rbrack$ to its spectrum $X \lbrack k \rbrack$ at the frequency $\omega_k = \frac{2 \pi k }{N}$. Since $X\lbrack k\rbrack$ repeats on intervals of length $N$, it suffices to consider the values of $X\lbrack k \rbrack$ at $N$ consecutive points $k = 0, 1, \cdots, N-1$.
The DFT is a one-to-one transformation, meaning that given $X \lbrack k \rbrack$, we can recover the original signal $x \lbrack n \rbrack$ using the inverse DFT (IDFT):
\begin{equation}
x \lbrack n \rbrack = \frac{1}{N} \sum^{N-1}_{k=0} X \lbrack k \rbrack e^{i 2\pi k n / N}.
\end{equation}

The DFT's significance lies in its application to signal processing algorithms, particularly within two important contexts. Firstly, the DFT operates on discrete inputs and produces discrete outputs, making it computationally suitable for digital signal processing. Secondly, the development of efficient algorithms, such as the Fast Fourier Transform (FFT) \cite{brigham1988fast}, has revolutionized DFT computation. The FFT exploits the symmetry properties of the DFT and employs a divide-and-conquer approach \cite{gorlatch1998programming}, recursively breaking down the DFT into smaller subproblems. This approach drastically reduces the computational complexity from $O(N^2)$ to $O(N\log N)$ \cite{rao2021global}. Notably, the inverse DFT, which exhibits a similar structure as the DFT, can also be efficiently computed using the inverse Fast Fourier Transform (IFFT). These advancements in DFT and FFT techniques have significantly enhanced the efficiency and practicality of signal-processing algorithms across various domains.

\subsection{Construct Dynamics in Fourier Domain}
% By delving into the frequency components of data in the Fourier domain, we can capture comprehensive patterns that extend across the entire dataset and uncover lasting connections between variables \cite{zhou2022fedformer}. % rao2021global
% In the Fourier domain, time series data is represented as a sequence of frequency components, each representing a unique frequency present in the data. Analyzing these frequency components enables the identification of dominant frequencies and patterns that may not be easily discernible in the time domain. Moreover, the frequency domain facilitates the exploration of relationships between different frequencies, unveiling long-term associations between variables that elude simplistic time-domain models.

% To leverage these insights, we introduce Fourier Ordinary Differential Equations (FODE) that learn dynamics in the Fourier domain. Given the input data $x$ and an implicit initial time $t_0$, its corresponding representation $X$ in the Fourier domain can be represented by:
Analyzing data in the Fourier domain allows us to capture patterns spanning the entire dataset and reveal long-term dependencies between variables \cite{zhou2022fedformer}. 
Here, time series are decomposed into frequency components, where dominant frequencies and cross-frequency interactions become more explicit than in the time domain. 
Building on this perspective, we introduce Fourier Ordinary Differential Equations (FODE) to learn dynamics directly in the frequency domain. 
Given input data $x$ and an implicit initial time $t_0$, its Fourier representation $X$ can be written as:
\begin{equation}
X(k, t_0) = \sum^{N-1}_{n=0} x(n, t_0) e^{-i 2\pi k n / N}.
\end{equation}

$X$ is the complex tensor and represents the spectrum of $x$. For real input $x \lbrack n \rbrack$, its DFT is conjugate symmetric \cite{rao2021global}, i.e. $X\lbrack N-k \rbrack = X^{*}\lbrack k \rbrack$. The reverse is true as well: if we perform IDFT to $X\lbrack k \rbrack$ which is conjugate symmetric, a real discrete signal can be recovered. This property implies that the half of the DFT $\{X \lbrack k \rbrack : 0 \leq k \leq \lceil \frac{N}{2} \rceil \}$ contains the full information about the frequency characteristics of $x\lbrack n \rbrack$. 
Suppose the real and imaginary parts of $X \lbrack k \rbrack$ are $X \lbrack k \rbrack ^{real}$ and $X \lbrack k \rbrack ^{imag}$, respectively. The complex numbers represent the spectrum of the signal in the Fourier domain, which provides information about both the amplitude and phase of the frequency components present in the original signal. The magnitude of the complex numbers represents the strength or magnitude of each frequency component, while the phase represents the phase shift or timing information associated with each component. 
We concatenate the $X \lbrack k \rbrack ^{real}$ and $X \lbrack k \rbrack ^{imag}$ together by 
\begin{equation}
    X^{info} = X \lbrack k \rbrack ^{real} \oplus X \lbrack k \rbrack ^{imag}, 
\end{equation}
where the $\oplus$ represents the concatenate symbol. Thus, $X^{info}$ contains the real and imaginary part information without the imaginary symbol $i$. We aim to learn a mapping $g: X \rightarrow Z$:
\begin{equation}
    Z^{info} = g(X^{info}).
\end{equation}
We do this by means of a basic neural network, i.e., $g(\cdot)$ is a basic neural network, which can be implemented by a Multilayer Perceptron (MLP) in practice.
% \cite{haykin1998neural} 

The obtained tensor $Z^{info}$ contains the information in the Fourier domain, thus we separate it to extract the real and ``imaginary" part by: 
\begin{equation}
    Z \lbrack k \rbrack ^{real} \oplus Z \lbrack k \rbrack ^{imag} = Z^{info}.
\end{equation}
Note that the $Z \lbrack k \rbrack ^{imag}$ does not contain the imaginary unit $i$, so we construct the complex tensor by applying an imaginary unit $i$ on the ``imaginary" part $Z \lbrack k \rbrack ^{imag}$ and obtain a real complex tensor $Z \lbrack k \rbrack = Z \lbrack k \rbrack ^{real} + i Z \lbrack k \rbrack ^{imag}$. The complex tensor $Z \lbrack k \rbrack$ can be seen as a representation in the Fourier space.
Finally, we can map back to the time domain by applying the inverse Fast Fourier Transform (IFFT):
\begin{equation}
  z \lbrack n \rbrack = \frac{1}{N} \sum^{N-1}_{k=0} Z \lbrack k \rbrack e^{i 2\pi k n / N}.
\end{equation}

\subsection{Fourier Ordinary Differential Equations}
Building on the Fourier dynamics, we define the dynamic function $f$ using the Fast Fourier Transform (FFT), a neural network $g(\cdot)$, and the Inverse FFT (IFFT). 
This function models changes in the Fourier space as a function of the data $x$ and an auxiliary time variable $t$, rather than the explicit timestamps in the data.
The system state at time $t_1$ is obtained by solving an initial value problem (IVP) with an ODE solver:
\begin{align}
x(t_1) \nonumber
&= x(t_0) + \int_{0}^{t_1} f_{\theta}(x, t)\,dt \\ 
&= \text{Solver}(x(t_0), f_{\theta}, t_0, t_1),
\label{ODEsolver}
\end{align}
where $x(t_0)$ is the initial state, $f$ the dynamic function, and $\theta$ the parameters of $g(\cdot)$. The ODE solver integrates forward from $x(t_0)$ to approximate the solution.

\label{subsec:fode}
\begin{definition}[Fourier ordinary differential equation (FODE)]
\label{def:fode}
Let $x_0\in\mathbb{R}^{N}$ be an initial time–series segment observed at (implicit) time
$t_0$.
A \emph{FODE} describes the evolution of a hidden state
$x:[t_0,t_1]\to\mathbb{R}^{N}$, with $x(t_0)=x_0$, through
\begin{equation}\label{eq:fode}
  \frac{\mathrm{d}x(t)}{\mathrm{d}t}=f_{\mathrm{FODE}}\bigl(x(t),t;\,\theta_g\bigr),
\end{equation}
where $f_{\mathrm{FODE}}$ is constructed in Fourier space:
\begin{align}
  f_{\mathrm{FODE}}\bigl(x,t;\theta_g\bigr)
    &=\operatorname{IFFT}\!\Bigl(
        \mathcal{M}\!
        \bigl(
          g\bigl(\mathcal{P}(\operatorname{FFT}(x)),\,t;\theta_g\bigr)
        \bigr)
      \Bigr).                             \label{eq:fode_impl}
\end{align}

$\operatorname{FFT},\operatorname{IFFT}$ denote the (inverse) fast Fourier transform; $\mathcal{P}$ concatenates $\Re$ and $\Im$ parts into one real vector ($X^{\text{info}}$); $g(\cdot,t;\theta_g)$ is a neural network that produces $Z^{\text{info}}$; $\mathcal{M}$ reconstructs the complex spectrum $Z[k]$, enforcing the conjugate‑symmetry constraint $Z[N-k]=\overline{Z[k]}$ so that the final $\operatorname{IFFT}$ yields a real signal.
Given~\eqref{eq:fode}, the state at $t_1$ is obtained with any ODE solver
\[
  x(t_1)
  =\text{ODESolver}\!\bigl(x(t_0),f_{\mathrm{FODE}},t_0,t_1;\theta_g\bigr).
\]
\end{definition}

\begin{lemma}[Lipschitz continuity of $f_{\mathrm{FODE}}$]
\label{lem:lipschitz_fode}
Assume (1) $g(\cdot,t;\theta_g)$ is $L_g$‑Lipschitz in its first argument for every fixed~$t$, and (2) $\mathcal{P}$ and $\mathcal{M}$ are $L_{\mathcal{P}}$‑ and $L_{\mathcal{M}}$‑Lipschitz, all with respect to the Euclidean norm.  
Because $\operatorname{FFT}$ and $\operatorname{IFFT}$ are bounded linear
operators, the composition in~\eqref{eq:fode_impl} is
$L_f$‑Lipschitz in $x$, where
\[
  L_f
  \;\le\;
  \bigl\lVert\operatorname{IFFT}\bigr\rVert\,
  L_{\mathcal{M}}\,
  L_g\,
  L_{\mathcal{P}}\,
  \bigl\lVert\operatorname{FFT}\bigr\rVert.
\]
\end{lemma}

\begin{proof}
Both FFT and IFFT are linear maps with bounded operator norms.
Lipschitz constants multiply under composition, so the bound above
follows directly.  Conjugate symmetry is preserved because
$\mathcal{M}$ constructs the missing half of the spectrum by complex
conjugation; this step is linear and thus Lipschitz with constant~1.
\end{proof}

\begin{theorem}[Existence and uniqueness]
\label{thm:wellposed_fode}
If the conditions of Lemma~\ref{lem:lipschitz_fode} hold, then
$f_{\mathrm{FODE}}(\cdot,t;\theta_g)$ is globally Lipschitz in~$x$ and
continuous in~$t$.  Hence, by the Picard–Lindelöf theorem, the IVP
\eqref{eq:fode} admits a unique solution
$x:[t_0,t_1]\to\mathbb{R}^{N}$.
\end{theorem}

\begin{proof}
Picard–Lindelöf (a.k.a.\ Cauchy–Lipschitz) applies directly once global
Lipschitz continuity in~$x$ and continuity in~$t$ are established.
\end{proof}

% \paragraph{Adjoint sensitivity.}
Because~\eqref{eq:fode} defines a reversible flow, gradients with respect to both the network parameters $\theta_g$ and the initial state $x(t_0)$
can be obtained efficiently via the adjoint method
\cite{pontryagin1961mathematical}, which solves an auxiliary ODE
backwards in time without having to store intermediate states.

\subsection{Element-Wised Filter} 
To enhance and refine the outcomes, we introduce an element-wise filter as a proposed approach. The filter is applied to the input $x(t_1)$ using a learnable filter matrix $K$. The operation is defined as follows:
\begin{equation}
\hat{x}(t_1) = K \odot x(t_1),
\end{equation}
where $\odot$ represents the element-wise multiplication, also known as the Hadamard product. The filter matrix $K$ has the same dimensions as $x(t_1)$ and acts as a filter for individual elements. The resulting vector $\hat{x}(t_1)$ represents the refined output. For exploration purposes without introducing biases, we initialize the filter matrix $K$ with a uniform distribution, enabling exploration of the solution space.
We show the pseudocode of our method in Algorithm \ref{Pseudocode_of_FODE}.

\begin{algorithm}[H]
    \caption{Pseudocode of FODE}
    % \textbf{Parameter}: $N_{train}$, $\alpha$, $\beta$, $T$\\
    % \textbf{Output}: Parameters $\theta_f$, $\theta_g$ 
    \begin{algorithmic}[] %[1] enables line numbers
    \STATE \textbf{Input}: $t_0$, $t_1$, data: $\lbrace x_n \rbrace = x_0, x_1, ... , x_{n}$
    \STATE \textbf{Parameters}: $W$, $\theta$
    \STATE {\color{brown}\textbf{Construct} \textit{f}: $x\lbrack n \rbrack \rightarrow z\lbrack n \rbrack$}
        \STATE $\qquad$ \textbf{(FFT)} $X \lbrack k \rbrack = \sum^{N-1}_{n=0} x\lbrack n \rbrack e^{-i 2\pi k n / N}$
        \STATE $\qquad$ $X \lbrack k \rbrack ^{real}, X \lbrack k \rbrack ^{imag} = real(X \lbrack k \rbrack), imag(X \lbrack k \rbrack)$
        \STATE $\qquad$ $X^{info} = X \lbrack k \rbrack ^{real} \oplus X \lbrack k \rbrack ^{imag}$
        \STATE $\qquad$ $Z^{info} = g(X^{info}, \theta_g)$
        \STATE $\qquad$ $Z \lbrack k \rbrack ^{real} \oplus Z \lbrack k \rbrack ^{imag} = Z^{info}$
        \STATE $\qquad$ $Z \lbrack k \rbrack = Z \lbrack k \rbrack ^{real} + i Z \lbrack k \rbrack ^{imag}$
        \STATE $\qquad$ \textbf{(IFFT)} $z \lbrack n \rbrack = \frac{1}{N} \sum^{N-1}_{k=0} Z \lbrack k \rbrack e^{i 2\pi k n / N}$
    \STATE $x(t_1) = ODESolver(x(t_0), {\color{brown}\textit{f}}, t_0, t_1, \theta)$
    \STATE \textbf{output}: $\hat{x}(t_1) = K \odot x(t_1)$
    \end{algorithmic}
    \label{Pseudocode_of_FODE}
\end{algorithm}

\section{Experiment}
% In this section, we conduct a comprehensive performance evaluation of the proposed Fourier Ordinary Differential Equations (FODE) model, comparing it with existing continuous models, namely NODE \cite{chen2018neural}, ANODE \cite{dupont2019augmented}, SONODE \cite{norcliffe2020second}, and NCDE \cite{kidger2020neural}, as well as discrete models including FNO \cite{li2020fourier}, RNN \cite{rumelhart1985learning} and LSTM \cite{hochreiter1997long}, on both time series forecasting and classification tasks.

% For the time series forecasting tasks, we first apply FODE on two synthesized periodic datasets. Then we evaluate the model on four real physical dynamic systems. 
% Accurate forecasting of physics dynamics is essential for understanding system behavior and making informed decisions. The chosen systems cover Unstable Oscillator, Forced Vibration, Lotka–Volterra System, and Glycolytic Oscillator. 

% % pyakillya2017deep
% Regarding the time series classification task, our focus is on Electrocardiogram (ECG) classification. ECG classification plays a crucial role in diagnosing and monitoring various cardiac conditions \cite{houssein2017ecg}. By analyzing the electrical activity of the heart captured in ECG signals, healthcare professionals can identify abnormalities and determine appropriate treatments. We employ three real ECG datasets, namely ECGFiveDays, ECG200, and ECG5000, which are obtained from \cite{bagnall2018uea}. We created an anonymous GitHub repository for method reproducibility, see \href{https://anonymous.4open.science/r/FODE_Anonymous_ICDM-6C2B/README.md}{Reproduction Code}.

In this section, we evaluate the proposed Fourier Ordinary Differential Equations (FODE) model against continuous models (NODE \cite{chen2018neural}, ANODE \cite{dupont2019augmented}, SONODE \cite{norcliffe2020second}, NCDE \cite{kidger2020neural}) and discrete models (FNO \cite{li2020fourier}, RNN \cite{rumelhart1985learning}, LSTM \cite{hochreiter1997long}) on time series forecasting and classification tasks.
For forecasting, we test FODE on two synthetic periodic datasets and four real physical systems: Unstable Oscillator, Forced Vibration, Lotka–Volterra, and Glycolytic Oscillator, where accurate prediction of dynamics is vital for system understanding and decision-making.
For classification, we focus on Electrocardiogram (ECG) signals, which are central to detecting cardiac abnormalities \cite{houssein2017ecg}. We use three real datasets—ECGFiveDays, ECG200, and ECG5000—from \cite{bagnall2018uea}. To ensure reproducibility, we provide an anonymous GitHub repository: \href{https://anonymous.4open.science/r/FODE_Anonymous_ICDM-6C2B/README.md}{Code}.

\subsection{Environment Setup}
% \cite{kingma2014adam} \cite{haykin1998neural} \cite{dormand1980family}
For all experiments, we utilize Adam as the optimizer with a learning rate of $10^{-3}$ and a batch size of $32$. We use the ReLU as the activate function. For all the ODE-based models, we used ``Dopri5" as the ODE solver. We trained each model $1000$ epochs.
We used the Mean Squared Error (MSE) as the loss function for time series forecasting tasks and Cross-Entropy Loss for classification tasks.
To ensure reliable results, we ran each experiment three times to account for experimental variability. The vector field in all the ODE-based models is parameterized using a 3-layer MLP. These three layers have the dimension of $(F, H)$, $(H, H)$, and $(H, F)$, respectively, where the $F$ represents the number of features and $H$ represents the hidden dimensions set as $H = 16$. For a fair comparison, we conduct the FNO with one Fourier layer where $modes=2$ and $width=8$. 
All the models were implemented in Python 3.9 and realized in PyTorch.
We employed a high-performance computing server equipped with NVIDIA A100-SXM4-80GB GPUs to train and evaluate all models and perform additional analysis.

\subsection{Periodic Time Series Forecasting}

\begin{table}[h]
\centering
\caption{Test MSE ($\times 10^{-5}$) of RNN, NODE, and our method across two periodic systems.}
\label{tab:Periodic Time Series}
\resizebox{\columnwidth}{!}{
\begin{tabular}{c |c|c c c}
\toprule
& Amp value & RNN & NODE & \textbf{FODE} \\
\toprule
\multirow{2}{*}{Periodic-3D-A} 
& 0.05  & $1.51 \pm 0.12 $ & $1.83 \pm 0.24$ & $\textbf{0.91} \pm \textbf{0.03}$ \\ 
\cmidrule{2-5}
& 0.10 & $2.10 \pm 0.21 $ & $3.20 \pm 0.32$ & $\textbf{0.42} \pm \textbf{0.02}$ \\
\midrule
\multirow{2}{*}{Periodic-3D-B} & 0.05  & $2.41 \pm 0.30 $ & $2.13 \pm 0.22$ & $\textbf{0.21}  \pm \textbf{0.02} $  \\ 
\cmidrule{2-5}
& 0.10  & $10.21 \pm 1.02$ & $0.51 \pm 0.11$ & $\textbf{0.20} \pm \textbf{0.01}$ \\
\bottomrule
\end{tabular}
}
\end{table}

To explore and verify the properties of FODE, we consider two synthetic, three-dimensional (3D) time-series datasets that exhibit periodic behavior with superimposed high-frequency waves. These datasets, referred to as Periodic-3D-A and Periodic-3D-B, are designed to evaluate the performance of predictive models under varying degrees of frequency and amplitude.

% \noindent \textbf{Periodic-3D-A.} This dataset is generated by sampling three channels over the time interval \(t \in [0, 20]\) at 1000 uniformly spaced points. The time series for each of the three components is defined as follows:
% $x(t) = \sin(t) + amp \times \sin(20t), \quad y(t) = \cos(t) + amp \times \cos(20t), \quad z(t) = \sin(2t) + amp \times \sin(20t)$.
% Here, the primary waveforms \( \sin(t) \) and \( \cos(t) \) capture low-frequency oscillations, while the additional terms with frequency \( 20 \, \text{rad/s} \) introduce high-frequency wave.
% The $amp$ is the amplitude of the high-frequency wave.
% The resulting 3D signal is partitioned into input-output subsequences, where each input consists of 10 time steps and each output consists of the subsequent 10 time steps. An 80\%-20\% chronological split is applied to create training and testing sets.

% \noindent \textbf{Periodic-3D-B.} Similarly constructed with \textbf{Periodic-3D-A}, this dataset differs in the fundamental frequencies of the components:
% $x(t) = \sin(2t) + amp \times \sin(20t), \quad y(t) = \cos(2t) + amp \times \cos(20t), \quad z(t) = \cos(5t) + amp \times \sin(20t).$
% The base waveforms involve frequencies of \( 2 \, \text{rad/s} \) and \( 5 \, \text{rad/s} \), with additional high-frequency components at \( 20 \, \text{rad/s} \).

\noindent \textbf{Periodic-3D.} 
This family of synthetic datasets is generated by sampling three channels over the time interval \(t \in [0, 20]\) at 1000 uniformly spaced points. Each dataset combines low-frequency base oscillations with an additional high-frequency component at \(20 \, \text{rad/s}\), scaled by an amplitude parameter $amp$. 
In \textbf{Periodic-3D-A}, the base waveforms are defined as
$
x(t) = \sin(t) + \text{amp} \times \sin(20t), \quad 
y(t) = \cos(t) + \text{amp} \times \cos(20t), \quad 
z(t) = \sin(2t) + \text{amp} \times \sin(20t),
$
where the fundamental frequencies correspond to \(1 \, \text{rad/s}\) for \(x(t), y(t)\) and \(2 \, \text{rad/s}\) for \(z(t)\). 
In \textbf{Periodic-3D-B}, the base waveforms are instead
$
x(t) = \sin(2t) + \text{amp} \times \sin(20t), \quad 
y(t) = \cos(2t) + \text{amp} \times \cos(20t), \quad 
z(t) = \cos(5t) + \text{amp} \times \sin(20t),
$
with fundamental frequencies of \(2 \, \text{rad/s}\) for \(x(t), y(t)\) and \(5 \, \text{rad/s}\) for \(z(t)\). 
In both datasets, the primary waveforms capture low-frequency dynamics while the additional terms introduce high-frequency oscillations. The resulting 3D signals are segmented into input-output subsequences, where each input consists of 10 time steps and each output consists of the subsequent 10 time steps. An 80\%-20\% chronological split is applied to create training and testing sets.
They simulate controlled, nontrivial periodic time-series data with known ground truth.

Table~\ref{tab:Periodic Time Series} summarizes the test Mean Squared Error (MSE) of RNN, NODE, and our method across two periodic systems with varying amplitude values. The results consistently demonstrate that our method outperforms the alternatives in capturing periodic trends and high-frequency components.

\begin{figure}[htb]
\centering
\includegraphics[width=0.8\linewidth]{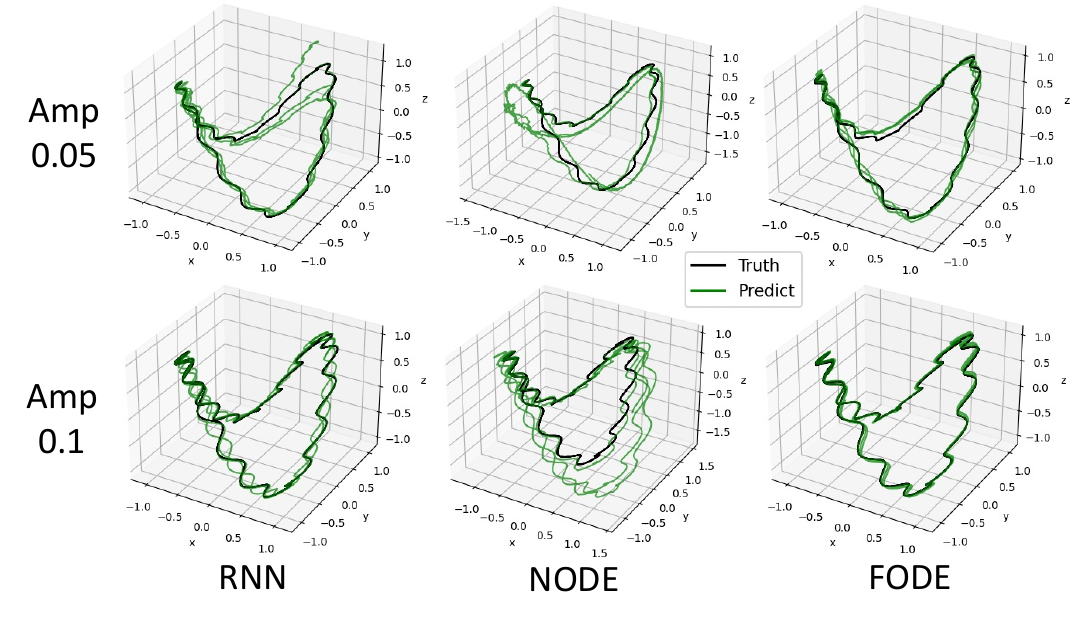}
\caption{Periodic-3D-A. Performance comparison of predictive models (RNN, NODE, and Ours) on the Periodic-3D-A dataset. The black lines represent the ground truth, while the green lines show predictions. The top row corresponds to a high-frequency amplitude of 0.05, and the bottom row corresponds to an amplitude of 0.1.}
\label{fig:Sine3D}
\end{figure}

% For Periodic-3D-A (Fig.~\ref{fig:Sine3D}), as the amplitude of the high-frequency component increases from 0.05 to 0.10, the MSE of RNN and NODE rises significantly, indicating reduced robustness under stronger high-frequency effects. Specifically, RNN's MSE increases from \(1.5 \times 10^{-5}\) to \(2.1 \times 10^{-5}\), while NODE's MSE escalates from \(1.8 \times 10^{-5}\) to \(3.2 \times 10^{-5}\). In contrast, our method demonstrates a significant reduction in MSE, decreasing from \(0.9 \times 10^{-5}\) to \(0.4 \times 10^{-5}\), underscoring its ability to handle high-frequency variations effectively.
For Periodic-3D-A (Fig.~\ref{fig:Sine3D}), increasing the high-frequency amplitude from 0.05 to 0.10 raises the MSE of RNN (from \(1.5 \times 10^{-5}\) to \(2.1 \times 10^{-5}\)) and NODE (from \(1.8 \times 10^{-5}\) to \(3.2 \times 10^{-5}\)), showing reduced robustness. In contrast, our method lowers MSE from \(0.9 \times 10^{-5}\) to \(0.4 \times 10^{-5}\), effectively handling high-frequency effects. For Periodic-3D-B (Fig.~\ref{fig:Cose3D}), degradation is sharper: RNN’s MSE jumps from \(2.4 \times 10^{-5}\) to \(10.2 \times 10^{-5}\). Our method again achieves the best results, maintaining \(0.2 \times 10^{-5}\) for both amplitudes.

\begin{figure}[htb]
\centering
\includegraphics[width=0.8\linewidth]{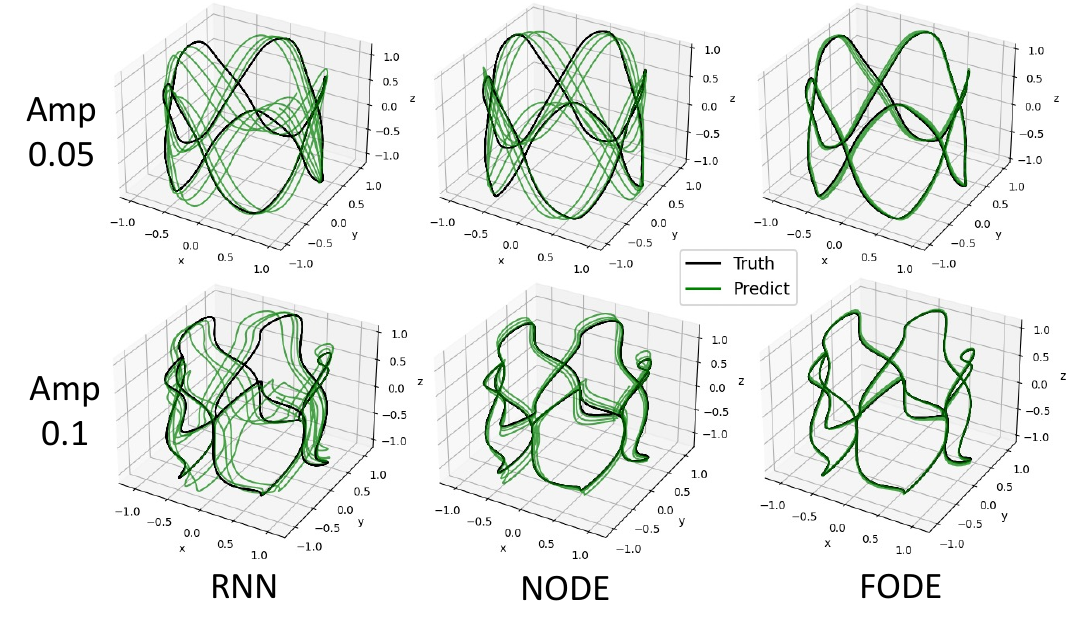}
\caption{Periodic-3D-B. Performance comparison of predictive models (RNN, NODE, and Ours) on the Periodic-3D-B dataset. The black lines represent the ground truth, while the green lines show predictions. The top row corresponds to a high-frequency amplitude of 0.05, and the bottom row corresponds to an amplitude of 0.1. }
\label{fig:Cose3D}
\end{figure}

% For Periodic-3D-B (Fig.~\ref{fig:Cose3D}), the performance degradation is more pronounced for RNN, with MSE increasing from \(2.4 \times 10^{-5}\) to \(10.2 \times 10^{-5}\). Our method consistently delivers the lowest MSE, achieving \(0.2 \times 10^{-5}\) for both amplitude settings, further highlighting its superior performance in capturing periodic trends amidst high-frequency components.

% These findings validate the robustness of our approach in disentangling periodic structures from high-frequency waves, outperforming traditional models under varying amplitude conditions. The synthetic datasets used here (Figs.~\ref{fig:Sine3D} and~\ref{fig:Cose3D}) emulate real-world scenarios where periodic behaviors often coexist with high-frequency variations.
% For example, a typical application is health monitoring systems, such as ECG signals, where high-frequency waves may arise from body movements or sensor artifacts, complicating signal interpretation. These scenarios are critical in fields such as physical systems, where accurate periodic modeling is essential.

These results confirm the robustness of our approach in separating periodic structures from high-frequency noise. The synthetic datasets (Figs.~\ref{fig:Sine3D},~\ref{fig:Cose3D}) mimic real scenarios—such as ECG monitoring, where motion or sensor artifacts introduce high-frequency waves—highlighting the importance of accurate periodic modeling in practice.

\subsection{Physical Dynamics Forecasting}
\label{Time Series Forecasting}

\begin{figure*}[htb]
\centering
\includegraphics[width=1.7\columnwidth]{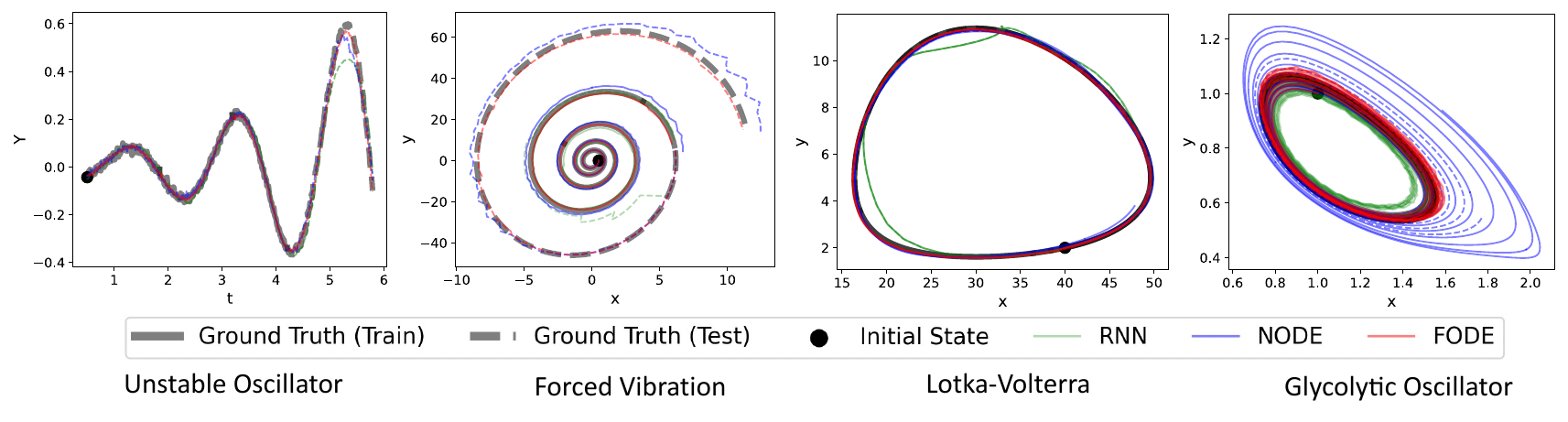}
\caption{Visual comparison of model predictions on four representative dynamical systems: Unstable Oscillator (far left), Forced Vibration (second from left), Lotka Volterra (third), and Glycolytic Oscillator (far right). The figure shows ground-truth training data (solid gray) and testing data (dashed gray), along with the initial state (black dot) and trajectories generated by RNN (green), NODE (blue), and FODE (red).}
\label{4systems}
\end{figure*}

We evaluate the performance of the proposed FODE method for time-series forecasting tasks on four real physical dynamic systems. We use multiple benchmark models including RNN, LSTM, NODE, ANODE, and SONODE.

\noindent \textbf{Unstable Oscillator.}
Unstable oscillators can arise in physical or biological contexts where a positive feedback mechanism causes the amplitude of oscillations to grow unbounded over time \cite{fradkov1998introduction}. This behavior is common in certain mechanical resonators or biological rhythms under persistent excitation. We construct a dataset that captures an exponentially growing oscillatory signal with additive noise:
\[
x(t) = 0.1\, e^{0.5t} \bigl(\cos(\pi t + 1) + \sin(\pi t - 1)\bigr) + \eta(t),
\]
where \(t \in [0,2\pi]\) is sampled at increments of 0.01, and \(\eta(t)\) is zero-mean Gaussian noise with standard deviation 0.01. As seen in Figure~\ref{4systems} (left panel), FODE accurately follows the rapid amplitude growth and captures both the exponential trend and the oscillatory phase shift more effectively than the alternative methods.

\noindent \textbf{Forced Vibration.}
Many engineered systems, such as bridges or vehicle suspensions, can experience forced vibrations when subjected to external periodic loading \cite{rao2019vibration}. We model this phenomenon using a second-order forced oscillator, which is converted into two first-order ODEs:
\[
\begin{cases}
    \dot{x} = v,\\
    \dot{v} = -2\,\zeta\,\omega_n \,v \;-\; \omega_n^2 \,x \;+\; F_0 \,\cos(\Omega t),
\end{cases}
\]
where \(\zeta = -0.1\) (negative damping ratio), \(\omega_n = 2\pi\) (natural frequency), \(F_0 = 0.1\) (forcing amplitude), and \(\Omega = 4.0\) (forcing frequency). 
We integrate from \(t = 0\) to \(t = 5\) using \(\Delta t = 0.01\), starting with \(x_0 = 0.5\) and \(v_0 = 0\). The negative damping drives the amplitude to grow over time, and Figure~\ref{4systems} (second panel) illustrates how FODE better tracks the outward spiral trajectory compared to other methods, which tend to deviate at larger radii.

\noindent \textbf{Lotka--Volterra System.}
The Lotka-Volterra equations are a pair of nonlinear ODEs that characterize the dynamics of predator-prey interaction \cite{yang2023latent}. They demonstrate how nonlinear feedback can produce sustained population cycles. We consider a Lotka-Volterra system:
\[
\begin{cases}
\dot{x} = \alpha x \;-\; \beta x\,y,\\
\dot{y} = \delta x\,y \;-\; \gamma y,
\end{cases}
\]
where \(\alpha = 0.1\), \(\beta = 0.02\), \(\gamma = 0.3\), and \(\delta = 0.01\). The system is integrated from \(t = 0\) to \(t = 100\) at 500 time points, starting from \(x_0 = 40\) and \(y_0 = 2\). As shown in Figure~\ref{4systems} (third panel), FODE preserves the correct phase and amplitude of these cyclical fluctuations, whereas some baselines drift substantially over longer forecasts.

\noindent \textbf{Glycolytic Oscillator.}
Glycolysis is a fundamental metabolic pathway in cells, and under certain conditions, it can exhibit rhythmic, oscillatory behavior due to feedback in enzymatic reactions. We use the same constants in the equations as in \cite{qian2022d}, given by
\[
\begin{cases}
\dot{x}_1 = a - b\,x_1 - x_1\,x_2^2,\\
\dot{x}_2 = b\,x_1 - x_2 + x_1\,x_2^2,
\end{cases}
\]
where \(x_1\) and \(x_2\) represent substrate and product concentrations, respectively, and \(a=0.75\), \(b=0.1\). 
The model is integrated from \(t = 0\) to \(t = 100\) with 1000 time points, starting at \(\bigl(x_1(0),\,x_2(0)\bigr) = (1.0,\,1.0)\). Figure~\ref{4systems} (right panel) demonstrates that FODE closely tracks the sustained oscillatory cycle, in contrast to other approaches that either underestimate or overestimate the oscillation radius.

Overall, FODE consistently outperforms the baselines on these four physical dynamics, achieving the lowest forecasting errors for each system, as summarized in Table~\ref{Tab_Testing_MSE_time_series_forecasting}. These results highlight the advantage of incorporating Fourier operators into neural ODE frameworks, enabling enhanced modeling of both oscillatory and unbounded dynamical behaviors.

\begin{table}[t]
\centering
\caption{Test MAPE (\%) on four dynamical systems.}
\resizebox{0.5\textwidth}{!}{
\begin{tabular}{lcccc}
\hline
 % & Unstable Oscillator & Forced Vibration & Lotka-Volterra & Glycolytic Oscillator \\
& Unstable Osc. & Forced Vib. & Lotka-Vol. & Glycolytic Osc. \\
\hline
RNN & 23.31 $\pm$ 0.12 & 23.45 $\pm$ 4.31 & 18.13 $\pm$ 5.34 & 5.12 $\pm$ 0.14 \\
LSTM & 16.58 $\pm$ 0.95 & 15.49 $\pm$ 1.01 & 10.12 $\pm$ 1.05 & 4.05 $\pm$ 1.59 \\
NODE & 13.41 $\pm$ 1.33 & 18.41 $\pm$ 3.12 & 2.45 $\pm$ 0.09 & 23.59 $\pm$ 5.09 \\
ANODE & 11.13 $\pm$ 1.94 & 19.14 $\pm$ 1.23 & 3.14 $\pm$ 1.24 & 30.14 $\pm$ 1.44 \\
SONODE & 12.34 $\pm$ 0.88 & 14.19 $\pm$ 4.24 & 2.36 $\pm$ 1.25 & 29.95 $\pm$ 2.45 \\
\hline
\textbf{FODE} & \textbf{8.98 $\pm$ 1.21} & \textbf{1.34 $\pm$ 0.61} & \textbf{1.87 $\pm$ 0.09} & \textbf{0.51 $\pm$ 0.04} \\
\hline
\end{tabular}
\label{Tab_Testing_MSE_time_series_forecasting}
}
\end{table}

\begin{table*}[htbp]
  \centering
  \caption{Test performance on time–series tasks.  
           Left block: forecasting (MAPE\,\%\,\,$\pm$\,std).  
           Right block: classification (MSE\,\,$\pm$\,std).}
  \label{tab:forecast_vs_classification}
  \small
  \setlength{\tabcolsep}{4.5pt}
  \resizebox{\textwidth}{!}{%
  \begin{tabular}{lccccc|ccc}
    \toprule
    & \multicolumn{5}{c|}{\textbf{Forecasting (MAPE\,\% $\pm$ std.)}} 
    & \multicolumn{3}{c}{\textbf{Classification (MSE $\pm$ std.)}} \\
    \cmidrule(lr){2-6}\cmidrule(lr){7-9}
    \textbf{Method} &
        Spanish Load & Building Load & Building Temp. & Spanish Temp. & ECG200 &
        ECGFiveDays & ECG200 & ECG5000 \\
    \midrule
        \textbf{FODE} & $\mathbf{0.83\pm0.05}$ & $\mathbf{(0.98\pm0.10)\times10^{-3}}$ & $\mathbf{6.41\pm0.29}$ & $\mathbf{7.16\pm0.24}$ & $\mathbf{12.48\pm0.41}$ &
        $\mathbf{0.114\!\pm\!0.072}$ & $\mathbf{0.320\!\pm\!0.015}$ & $\mathbf{0.211\!\pm\!0.029}$ \\
    NODE   & $3.06\pm0.10$ & $(1.00\pm0.05)\times10^{-2}$ & $14.95\pm0.33$ & $9.47\pm0.21$ & $35.09\pm0.70$
           & $0.192\!\pm\!0.034$ & $0.377\!\pm\!0.005$ & $0.254\!\pm\!0.010$ \\
    ANODE  & $3.18\pm0.10$ & $(2.54\pm0.08)\times10^{-3}$ & $6.87\pm0.23$ & $8.71\pm0.22$ & $17.43\pm0.35$
           & $0.189\!\pm\!0.027$ & $0.393\!\pm\!0.021$ & $0.253\!\pm\!0.008$ \\
    SONODE & $2.09\pm0.07$ & $(2.00\pm0.05)\times10^{-2}$ & $11.15\pm0.36$ & $7.41\pm0.27$ & $30.61\pm0.68$
           & $0.178\!\pm\!0.040$ & $0.352\!\pm\!0.003$ & $0.247\!\pm\!0.001$ \\
    RNN    & $1.61\pm0.05$ & $(1.00\pm0.05)\times10^{-2}$ & $14.83\pm0.33$ & $23.64\pm0.44$ & $12.61\pm0.37$
           & $0.538\!\pm\!0.001$ & $0.583\!\pm\!0.002$ & $0.383\!\pm\!0.015$ \\
    LSTM   & $1.69\pm0.05$ & $(1.24\pm0.04)\times10^{-3}$ & $21.59\pm0.49$ & $29.78\pm0.64$ & $25.43\pm0.51$
           & $0.523\!\pm\!0.044$ & $0.605\!\pm\!0.010$ & $0.318\!\pm\!0.018$ \\
    FNO    & $2.94\pm0.09$ & $(6.10\pm0.20)\times10^{-1}$ & $27.90\pm0.67$ & $10.60\pm0.20$ & $99.43\pm4.04$
           & $0.170\!\pm\!0.101$ & $0.330\!\pm\!0.011$ & $0.269\!\pm\!0.020$ \\
    NCDE   &   -- & -- & -- & -- & -- 
           & $0.714\!\pm\!0.022$ & $0.623\!\pm\!0.013$ & $0.913\!\pm\!0.004$ \\
    \bottomrule
  \end{tabular}}
\end{table*}

\subsection{Time Series Classification}

We evaluate FODE on time series classification, comparing it with the same baselines as in Section~\ref{Time Series Forecasting}, along with NCDE~\cite{kidger2020neural} and FNO~\cite{li2020fourier}. 
Experiments are conducted on three standard ECG datasets that vary in length, categories, and underlying conditions.

\noindent \textbf{ECGFiveDays.} This dataset~\cite{hu2013time} contains ECG signals from a 67-year-old male recorded on two dates, five days apart, forming two classes. The task is to distinguish between the sessions based on subtle temporal variations. FODE achieves the lowest MSE (Table~\ref{tab:forecast_vs_classification}), showing strong sensitivity to these differences.

\noindent \textbf{ECG200.} ECG200~\cite{olszewski2001generalized} consists of single-heartbeat signals labeled as normal or myocardial infarction, providing a binary classification problem. Despite short sequences, successful modeling requires capturing morphological patterns. FODE performs competitively, and notably the variant without filter $K$ achieves the lowest MSE (Table~\ref{tab:forecast_vs_classification}).

\noindent \textbf{ECG5000.} Derived from a 20-hour recording of a heart failure patient in CHFDB~\cite{goldberger2000physiobank}, ECG5000 contains 5,000 normalized beats of length 140, categorized into five classes. Its diversity and multiple categories make the classification task more challenging compared to ECGFiveDays and ECG200.

\subsection{Time Series Forecasting}
% \noindent \textbf{Load and Temperature Forecasting.} 
We also test our model on two load datasets (Spanish Load and  Building Load) and two temperature datasets (Spanish Temp. and Building Temp.), each exhibiting strong periodic patterns. These datasets, collected from public sources such as Kaggle, pose forecasting challenges due to daily and seasonal cycles, noise, and occasional outliers.

The Spanish Load dataset contains hourly power demand records for Spain, while the Building Load and Building Temperature datasets capture consumption and indoor climate readings from commercial buildings. The Spanish Temperature dataset tracks nationwide meteorological conditions. Across all four datasets, FODE consistently outperforms baseline methods, achieving the lowest MAPE and standard deviation (see Table~\ref{tab:forecast_vs_classification}). This indicates FODE’s ability to model periodic structures and temporal dependencies effectively, even in the presence of noise and heterogeneity.

\begin{figure}[t]
\centering 
\includegraphics[width=0.9\columnwidth]{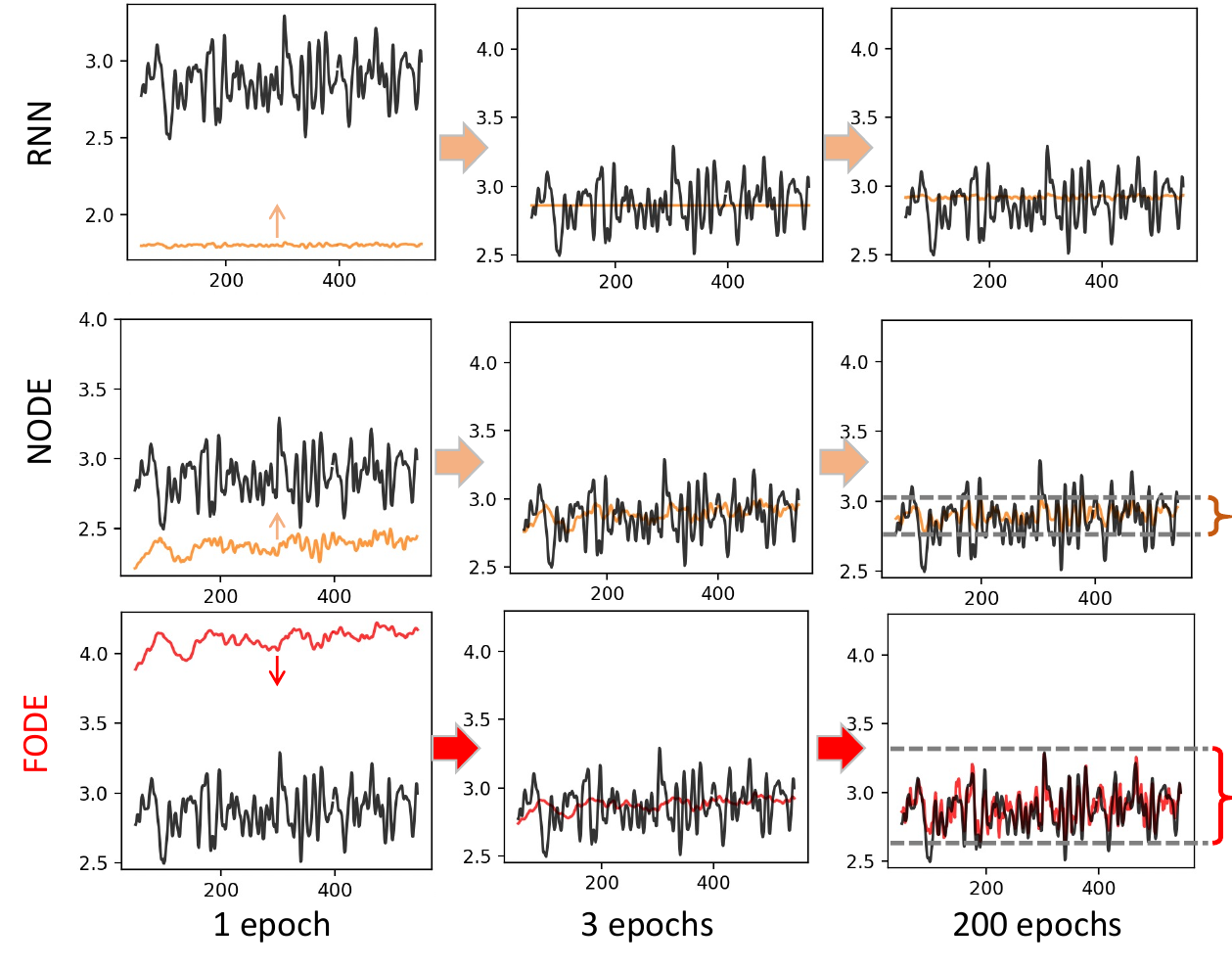}
\caption{Predicted trajectories (colored) vs. ground truth (black) at 1, 3, and 200 epochs. RNN fails to capture the global pattern. NODE improves learning but shows convergence bias. FODE quickly recovers both local details and the global pattern, achieving stable, accurate predictions.}
\label{Forecasting_Process}
\end{figure}

Figure~\ref{Forecasting_Process} illustrates the learning process of RNN, NODE, and FODE across training epochs. At early stages (1 epoch), RNN fails to learn meaningful dynamics, producing flat outputs. NODE starts capturing trends but suffers from convergence to biased regions. In contrast, FODE exhibits rapid adaptation, initially producing a rough global trend and subsequently refining local details. By 200 epochs, FODE not only matches the ground truth closely but also maintains amplitude fidelity and trend consistency—underscoring its strength in capturing both global patterns and fine-grained variations in time series.

\subsection{Hidden State Analysis of FODE}
To gain deeper insights into how FODE processes signals, we perform a short-time Fourier transform (STFT) \cite{griffin1984signal} on the model’s hidden states at various training epochs. By examining the hidden states in the frequency domain, we can observe how FODE dynamically transforms input signals throughout the training process.

Figure~\ref{freq_change} shows four representative samples, each row corresponding to a different input signal. Within each row, the columns display the STFT results for that signal’s hidden state as training progresses from left to right. Initially, the spectrogram exhibits a broad distribution of energy across different frequency bands. As the model learns, the energy appears to reorganize or concentrate in specific frequency regions, reflecting how FODE re-weights and reshapes the frequency content to optimize its predictive objective. By the later epochs (rightmost plots), the hidden states exhibit more refined frequency profiles, suggesting that the model converges to specialized representations for each sample.

\begin{figure}[t]
\centering 
\includegraphics[width=1\columnwidth]{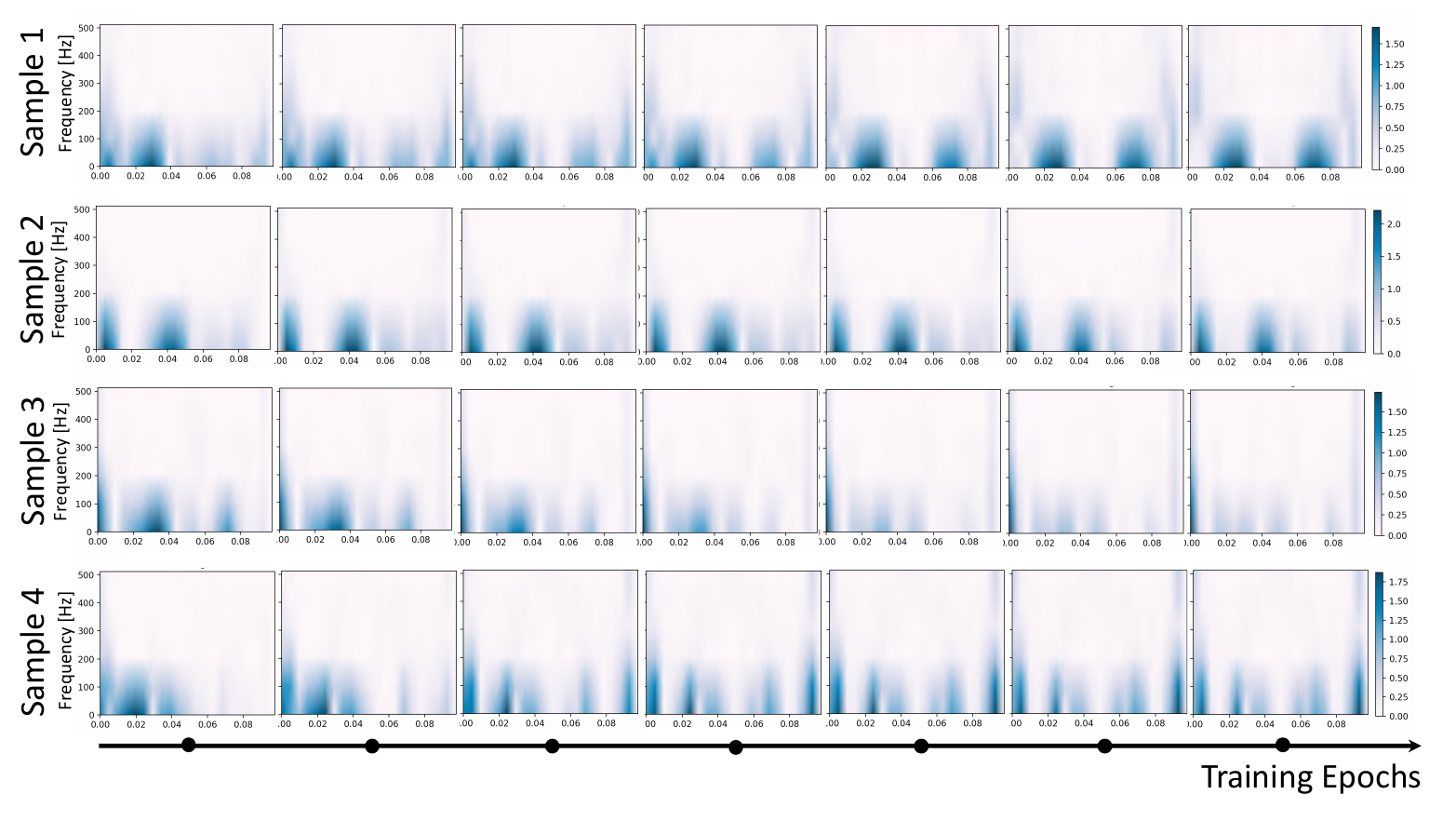}
\caption{Short-time Fourier transform (STFT) applied to FODE’s hidden state over training epochs, for four sample signals (one row per sample). From left to right in each row, the spectrogram evolves as the model trains, illustrating how FODE re-weights and reshapes frequency components in its hidden representation.}
\label{freq_change}
\end{figure}

\subsection{Ablation Study}

\begin{table}[htbp]
  \centering
  \caption{Ablation study: effect of introducing $K$ in FODE. The last column shows the relative reduction obtained with FODE.}
  \label{tab:ablation_fode}
  \small
  \setlength{\tabcolsep}{7pt}
  \resizebox{1\linewidth}{!}{
  \begin{tabular}{l l c c c}
    \toprule
    \textbf{Application} & \textbf{Dataset} & \textbf{FODE w/o $K$} & \textbf{FODE} & \textbf{MAPE\,Decrease} \\
    \midrule
    Power   & Spanish Load      & 2.36\%  & 0.83\%  & \textcolor{brown}{64.83\% $\downarrow$} \\
    Weather & Building Temp     & 7.95\%  & 6.41\%  & \textcolor{brown}{19.37\% $\downarrow$} \\
    Health  & ECG200            & 14.30\% & 12.48\% & \textcolor{brown}{12.73\% $\downarrow$} \\
    Physics & Forced Vib.  & 1.39\%  & 1.34\%  & \textcolor{brown}{\phantom{0}3.60\% $\downarrow$} \\
    \bottomrule
  \end{tabular}
  }
\end{table}

To assess the contribution of the Fourier-domain filter \(K\) in our model, we conduct an ablation study by comparing FODE with and without 
\(K\) across diverse domains, including power systems, weather forecasting, healthcare, and physical dynamics. Table~\ref{tab:ablation_fode} reports the Mean Absolute Percentage Error (MAPE) for both variants and the corresponding relative improvement.

The results demonstrate that introducing \(K\) improves performance across all tasks. The most notable gain is observed on the Spanish Load dataset, where MAPE drops by over $64\%$. Similarly, meaningful improvements are seen in the Building Temperature and ECG200 datasets, highlighting the role of \(K\) in capturing complex periodic and morphological patterns. Even in the physics domain (Forced Vibration), a modest improvement is observed, indicating \(K\)’s broad applicability. The ablation confirms that the learnable filter \(K\) enhances the model’s ability to adaptively suppress noise and emphasize relevant frequency components, thereby refining both the global structure and local details of time series predictions.

% \subsubsection{Effectiveness of K.}

% \begin{table}[htbp]
% \centering
% \caption{Ablation Study: Test MSE with and without $K$ in FODE}
% \renewcommand{\arraystretch}{1.3}
% \resizebox{0.95\linewidth}{!}{
% \begin{tabular}{c|c|c|c|c}
% \hline
% \textbf{Application} & \textbf{Dataset} & \textbf{FODE (w/o $K$)} & \textbf{FODE} & \textbf{MSE} \\
% \hline
% Power   & Spanish Load      & 1.02e-2 ± 6.5e-4 & 6.39e-3 ± 4.2e-4 & \textcolor{blue}{$\downarrow$~37.2\%} \\
% % Weather        & Building Temp     & 11.81 ± 0.58     & 3.23 ± 0.25      & \textcolor{blue}{$\downarrow$~72.7\%} \\
% Weather         & Spanish Temp      & 8.37 ± 0.47      & 5.57 ± 0.31      & \textcolor{blue}{$\downarrow$~33.4\%} \\
% Health & ECG200            & 1.16e-1 ± 4.9e-3 & 1.11e-1 ± 6.1e-3 & \textcolor{blue}{$\downarrow$~4.2\%} \\
% Physics & Forced Vibration  & 1.66e-2 ± 1.3e-3 & 1.47e-2 ± 1.1e-3 & \textcolor{blue}{$\downarrow$~11.4\%} \\
% \hline
% \end{tabular}
% }
% \label{tab:ablation_k}
% \end{table}

\subsection{Evolution of filter K with different initialization}
We investigate how the filter \(K\) evolves over the course of training when initialized in three ways: (1) all-zero entries, (2) all-one entries, and (3) Xavier uniform initialization \cite{glorot2010understanding}. Figure~\ref{Filter_K_change} provides a color-coded visualization of the per-epoch changes in \(K\). Each subfigure shows the evolving weight values (horizontal axis) and the corresponding training loss (vertical axis), enabling a direct comparison of how different initializations influence the filter’s learning dynamics.

As shown in Figure~\ref{Filter_K_change}, the zero initialization keeps weights near zero in early epochs, requiring steeper adjustments before convergence. Starting from ones biases the model toward uniformly positive weights, which can lead to large updates as training progresses. By contrast, Xavier initialization balances positive and negative initial values, yielding more gradual color shifts and suggesting a smoother training trajectory. Despite these differences, all three strategies successfully adapt \(K\) to reduce the loss, highlighting the model’s flexibility in incorporating different initialization schemes.

\begin{figure}[htb]
\centering
\includegraphics[width=1\columnwidth]{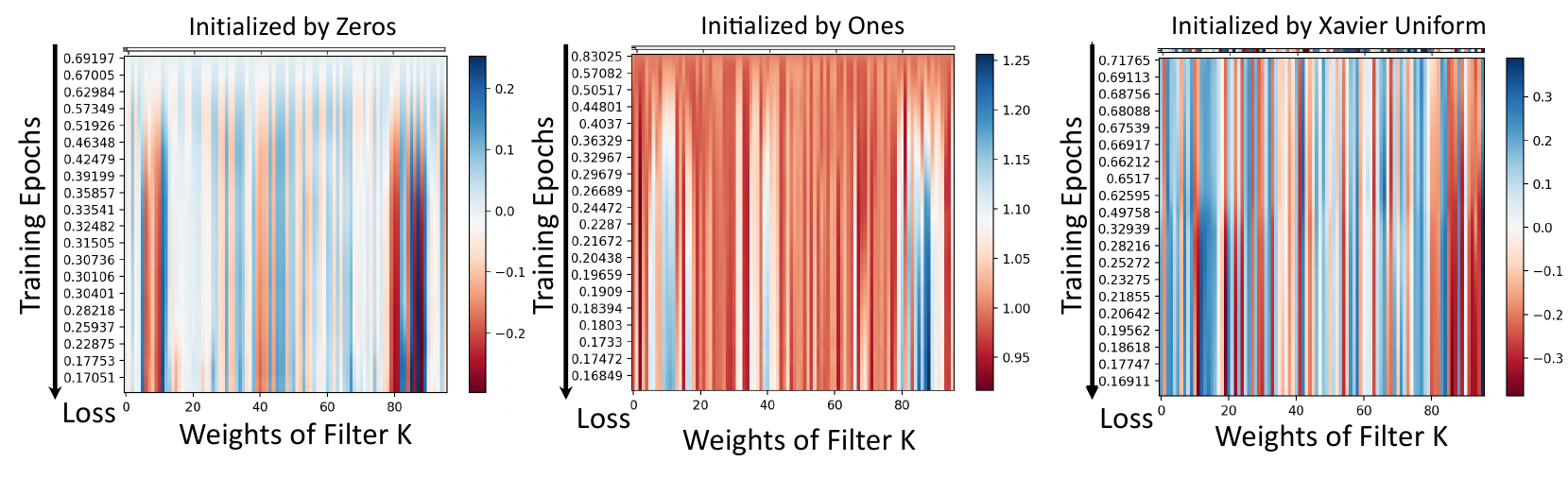}
\caption{Evolution of the filter \(K\) over training epochs, under three different initialization schemes: zeros (left), ones (middle), and Xavier uniform (right). 
Rows correspond to training iterations, with the vertical axis showing the loss (from higher at the top to lower at the bottom). The color represents the filter weight values, transitioning from negative (blue) to positive (red).
}
\label{Filter_K_change}
\end{figure}

\section{Conclusion}
In this work, we introduced FODE, an ODE-based model that leverages the Fourier domain to learn dynamics and enhance the representation of time series data. By operating in the Fourier domain, FODE can effectively capture underlying periodic patterns, surpassing the capabilities of traditional continuous models. The incorporation of an element-wise filter maintains granularity while enabling generalization. Experimental evaluations on various time series datasets demonstrated the superior performance of FODE.

\bibliography{IEEEexample}
\bibliographystyle{IEEEtran}

\end{document}